\documentclass[pdflatex,sn-mathphys-num]{sn-jnl}

\usepackage{graphicx}%
\usepackage{multirow}%
\usepackage{amsmath,amssymb,amsfonts}%
\usepackage{amsthm}%
\usepackage{mathrsfs}%
\usepackage[title]{appendix}%
\usepackage{xcolor}%
\usepackage{textcomp}%
\usepackage{manyfoot}%
\usepackage{booktabs}%
\usepackage{algorithm}%
\usepackage{algorithmicx}%
\usepackage{algpseudocode}%
\usepackage{listings}%
\usepackage{enumerate}

\theoremstyle{thmstyleone}%
\newtheorem{theorem}{Theorem}
\newtheorem{proposition}[theorem]{Proposition}%
\newtheorem{lemma}[theorem]{Lemma}
\newtheorem{corollary}[theorem]{Corollary}

\theoremstyle{thmstyletwo}%
\newtheorem{example}{Example}%
\newtheorem{remark}{Remark}%

\theoremstyle{thmstylethree}%
\newtheorem{definition}{Definition}%

\begin{document}

\title[An order-oriented approach to scoring hesitant fuzzy elements]{An order-oriented approach to scoring hesitant fuzzy elements}


\author[1,2]{\fnm{Luis} \sur{Merino}}\email{lmerino@ugr.es}
\equalcont{These authors contributed equally to this work.}

\author*[1,2]{\fnm{Gabriel} \sur{Navarro}}\email{gnavarro@ugr.es}

\author[1]{\fnm{Carlos} \sur{Salvatierra}}\email{carlosss@ugr.es}
\equalcont{These authors contributed equally to this work.}

\author[1,2]{\fnm{Evangelina} \sur{Santos}}\email{esantos@ugr.es}
\equalcont{These authors contributed equally to this work.}

\affil*[1]{\orgdiv{Department of Algebra}, \orgname{University of Granada}, \orgaddress{ \city{Granada}, \postcode{18071}, \country{Spain}}}

\affil*[2]{\orgdiv{IMAG}, \orgname{University of Granada}, \orgaddress{ \city{Granada}, \postcode{18071}, \country{Spain}}}

\abstract{Traditional scoring approaches on hesitant fuzzy sets often lack a formal base in order theory. This paper proposes a unified framework, where each score is explicitly defined with respect to a given order. This order-oriented perspective enables more flexible and coherent scoring mechanisms. We examine several classical orders on hesitant fuzzy elements, that is, nonempty subsets in $[0,1]$, and show that, contrary to prior claims, they do not induce lattice structures. In contrast, we prove that the scores defined with respect to the symmetric order satisfy key normative criteria for scoring functions, including strong monotonicity with respect to unions and the Gärdenfors condition. 

Following this analysis,  we introduce a class of functions, called dominance functions, for ranking hesitant fuzzy elements. They aim to compare hesitant fuzzy elements relative to control sets incorporating minimum acceptability thresholds. Two concrete examples of dominance functions for finite sets are provided: the discrete dominance function and the relative dominance function. We show that these can be employed to construct fuzzy preference relations on typical hesitant fuzzy sets and support group decision-making.}




\maketitle

\section{Introduction}

Hesitant fuzzy sets (HFSs), introduced by Torra \cite{Torra2009, Torra2010}, have proven to be a valuable tool to handle situations characterized by uncertainty and hesitation. HFSs are especially effective in contexts like multi-criteria decision making. Unlike traditional fuzzy sets (FSs) \cite{Zadeh1965}, which represent membership degrees with single values, HFSs allow a set of possible values as the membership degree. Such a set of possible values is called a hesitant fuzzy element (HFE), which is a nonempty subset of \([0,1]\). Nevertheless, unlike FSs, the lack of a total order in the class of all possible membership degrees complicates the comparison between several objects. Although some partial orders, as those introduced in \cite{JMNS22}, give some options for this task, these do not solve this problem completely. In this context, scores came out as an alternative. 

Since their initial proposal by Xu and Xia \cite{XuXia11}, and subsequent refinements by Farhadinia \cite{Farhadinia2013, Farhadinia14, Farhadinia16}, they have become as essential tools across a broad range of practical applications. Originally, scores were considered for  defininng systematic ranking methods over finite subsets of \([0,1]\), known as typical hesitant fuzzy elements (THFEs) \cite[Definition 3.1]{Bedregal-neg}. Nevertheless, recent advances by Torra and Narukawa \cite{Narukawa21, Narukawa22} have further expanded scores applications beyond the typical cases, introducing methods for aggregating hesitant information through generalized integrals. Notable applications of scores also include clustering analysis \cite{Chen13}, multi-criteria decision making \cite{Mishra21, Thillai16, Wang2022, Wei2016}, and group decision-making processes \cite{Chen2013, XuXia011, XuXia13, Yu13}.

It is important to distinguish between ranking methods (also known as preference or dominance relations~\cite{Davis1972, Nitzan1976}) and real order relations on HFEs. Orders on HFEs are algebraic structures that satisfy reflexivity, antisymmetry, and transitivity. Such orders have been extensively developed and discussed in works such as~\cite{Bedregal-neg, Bedregal14, Garmendia17, Li2017, Matze, Perez15, Wang16, Wang2016, XuXia2011, XuXia11, Zhang16}. In contrast, ranking methods are relations that may lack antisymmetry or transitivity, and are thus not proper orders. As shown in examples in \cite{Chen2013, Farhadinia13, Farhadinia16, Zhang2018HFPR, Zhu2014}, most ranking techniques are implicitly based on scores. For instance, the combination of mean and deviation scores proposed in~\cite{Chen2014, Liao2014, Wang19}, or the lexicographic method in~\cite[Section 3]{Farhadinia16}, illustrate how score-based approaches often lead to comparisons that do not induce a total or even a partial order. This has generated ambiguities regarding the theoretical status of such comparisons, as previously noted in methodological discussions such as~\cite[Subsection~8.4]{Rodriguez2015} and~\cite[Example~4]{Stateofart_HFS}. In this paper, we aim to clarify this problem by explicitly linking the notion of score with well-defined orders on HFEs.

On the other hand, despite their numerous applications, it is worth noting that most established definitions of HFE scores rely on a \emph{mean value interpretation}~\cite[Section 4]{Alcantud}, \cite{Campion}. From our perspective, this assumption may limit their applicability to a single evaluative framework. In this paper, we propose an alternative approach, that relaxes it, by defining HFE scores relative to a given order. This order-based perspective preserves the mean-value interpretation in context where it is suitable, while, simultaneously, extends the evaluative scope of scores. In doing so, it offers a greater flexibility in ranking and decision-making scenarios where traditional mean-based scores may  be insufficient. In this sense, we introduce dominance  functions in order to ranking objects not necessarily following the mean-value principle.

The paper is structured as follows. Section~2 provides a comprehensive review of various orders for HFEs, revisiting established definitions with a focus on typical HFEs. It also examines the importance of the symmetric order introduced in \cite{JMNS22}, its associated lattice structure, and presents new results for the typical case. Section~3 reexamines the concept of scores for HFEs, explicitly demonstrating how their properties can be deduced according to the associated order. This section further highlights the relevance of defining scores with respect to an order that induces a lattice structure. In Section~4, we show that scores related to the symmetric order can be characterized based on the main admissible properties outlined by Alcantud et al.~in \cite[Section 4]{Alcantud}. In Section~5, we introduce a new class of functions called dominance functions, which, essentially,  becomes a family of scores under certain properties. In particular, we highlight the role of dominance functions as a method to incorporate baseline requirements into the evaluation process, ensuring that scores reflect not only relative preferences but also compliance with threshold constraints. We then relate these scores to the construction of fuzzy preference relations with respect to an order. To illustrate these ideas, we also propose an example within the context of project evaluation. Finally, in Section 6, we focus on the conclusions.

\section{Orders on hesitant fuzzy elements} \label{StateOrders}

This section reviews the main orders on HFEs, highlighting essential relationships and distinctions among them. For consistency with prior works, we adopt the following set notation \cite{Alcantud2016, Alcantud}:

\begin{itemize}
    \item \( P([0,1]) \) denotes the family of all subsets of \([0,1]\).
    \item \( P^*([0,1]) \) denotes the family of all nonempty subsets of \([0,1]\).
    \item \( F([0,1]) \) denotes the family of all finite subsets of \([0,1]\).
    \item \( F^*([0,1]) \) denotes the family of all nonempty finite subsets of \([0,1]\).
\end{itemize}

A poset or partial order is defined as a set \( P \) equipped with a binary relation \( \leq \) that satisfies reflexivity, antisymmetry, and transitivity. The poset \( P \) is said to be bounded if it includes minimum and maximum elements, denoted by \( 0 \) and \( 1 \), such that \( 0 \leq p \leq 1 \) for every \( p \in P \). A poset is further characterized as a total or linear order if, for every pair of elements \( a, b \in P \), either \( a \leq b \) or \( b \leq a \) holds. An order \( \leq_2 \) is a \emph{refinement} of another order \( \leq_1 \) on \(P\) if \( a \leq_1 b \) implies \( a \leq_2 b \) for all \( a, b \in P \).

\subsection{Orders on typical hesitant fuzzy elements}

A prominent area of study involves finite subsets of \([0,1]\), \( F^*([0,1]) \), often referred to as typical hesitant fuzzy elements (THFEs) \cite[Definition 3.1]{Bedregal-neg}. The first-defined order for THFEs is the list order, see for instance \cite{Farhadinia13} . To introduce this order with some formalism, we may associate to each finite subset of \([0,1]\) with cardinality \( n \) an unique ordered tuple in \( [0,1]^n \),
\[
F^*([0,1]) \rightarrow \bigcup_{n \in \mathbb{N}} [0,1]^n,
\]
where \( A \mapsto A^* = (a_1, \dots, a_n) \), with \( a_1 < \dots < a_n \). The list order is then defined by:
\[
A \leq_{\text{list}} B \iff A = \{0\}, \text{ or } B = \{1\}, \text{ or } (\#A = \#B \text{ and } A^* \leq_{\text{prod}} B^*),
\]
where \( \leq_{\text{prod}} \) denotes the product order in \( [0,1]^n \),
\[
(x_1, \dots, x_n) \leq_{\text{prod}} (y_1, \dots, y_n) \iff x_1 \leq y_1, \dots, x_n \leq y_n.
\]

Despite its simplicity, this order plays a central role as it serves as the foundation for the, so called, admissible orders for THFEs \cite{Baz2024, Matzenauer2020, Matze, Matzenauer2021, Wang16}. However, it only compares sets with the same cardinality. To overcome this drawback, some refinements have been considered. An approach is the  \(\beta\)-normalization process \cite[Section~3.1]{Bedregal14} that consists in adding more elements to the set with fewer elements, so that both transformed sets can be compared with respect to the list order. Two of the most prominent \(\beta\)-normalization induced orders are the pessimistic order (also referred to as the Xu-Xia order \cite{Bedregal-neg, Santos2014, SantosBredegal}) and the optimistic order, both originally introduced in \cite{XuXia2011}. We formalize them as follows. 

Consider a set \( A \in F^*([0,1]) \) of cardinality \( r = \#A \), with ordered list \( A^* = (a_1, \dots, a_r) \). For any \( n \geq r \), we set
\[
A_{(n)} = (\underbrace{a_1, \dots, a_1}_{n - r}, a_1, a_2, \dots, a_r)\text{ and } A^{(n)} = (a_1, \dots, a_r, \underbrace{a_r, \dots, a_r}_{n - r}).
\]
Then, the \emph{pessimistic order} is defined as
\[
A \leq_{\text{pes}} B \iff A_{(n)} \leq_{\text{prod}} B_{(n)},\quad \text{for } n = \max\{\#A, \#B\},
\]
and the \emph{optimistic order} is given by
\[
A \leq_{\text{opt}} B \iff A^{(n)} \leq_{\text{prod}} B^{(n)}, \quad \text{for } n = \max\{\#A, \#B\},
\]
for any $A,B\in F^*([0,1])$. Here $\leq_{\text{prod}} $ denotes de product order.

\begin{example}
Let \( A = \{0.2, 0.4\} \) and \( B = \{0.3, 0.5, 0.8\} \). Since \( \#A = 2 < 3 = \#B \), we extend \(A\). The pessimistic extension is \( A_{(3)} = (0.2, 0.2, 0.4) \), and the optimistic extension is \( A^{(3)} = (0.2, 0.4, 0.4) \). We have both \( A_{(3)} \leq_{\text{prod}} B_{(3)} \) and \( A^{(3)} \leq_{\text{prod}} B^{(3)} \) , so \( A \leq_{\text{pes}} B \) and \( A \leq_{\text{opt}} B \).
\end{example}

Given a subset \( A \subseteq [0,1] \), we denote by \( 1-A \) the set \( \{1-a \mid a \in A\} \).

\begin{proposition}\label{propList}
The relations \( \leq_{\text{pes}} \) and \( \leq_{\text{opt}} \) are partial orders on \( F^*([0,1]) \) that refine the list order \( \leq_{\text{list}} \), and have minimum element \( \{0\} \) and maximum element \( \{1\} \). Moreover, for any \( A, B \in F^*([0,1]) \),
\[
A \leq_{\text{pes}} B \iff 1-A \leq_{\text{opt}} 1-B.
\]
\end{proposition}
\begin{proof}
All the poset properties are straightforward except the transitivity. This follows from the equivalence of the following statements:
\begin{enumerate}[i)]
\item \( A_{(n)} \leq_{\text{prod}} B_{(n)} \) for \( n = \max\{\#A, \#B\} \),
\item \( A_{(k)} \leq_{\text{prod}} B_{(k)} \) for every \( k \geq \max\{\#A, \#B\} \), and
\item \( A_{(k)} \leq_{\text{prod}} B_{(k)} \) for some \( k \geq \max\{\#A, \#B\} \).
\end{enumerate}
\end{proof}

Note, however, that these orders do not induce a lattice structure, as shown in Example \ref{ex1}. Although some prior work, such as \cite[Proposition 1]{Santos2014}, might suggest otherwise, a careful analysis reveals that the existence of infima and suprema is not guaranteed.

\begin{example}\label{ex1}
Consider the sets $A=\{0.1, 0.4, 0.5, 0.7\}$ and $B=\{0.3 , 0.6\}$ and the pessimistic order. It is easy to check that the greatest lower bound of $A$ and $B$ with three or less elements is $\{0.1, 0.3, 0.6\}$. This is not an infimum of $A$ and $B$ because it is lower than $\{0.1, 0.2, 0.3, 0.6\}$, which is also a common lower bound.
If $X$ is a common lower bound of $A$ and $B$ bigger than $\{0.1, 0.2, 0.3, 0.6\}$, then  $X$ needs to have more than three elements and if $X^*=(x_1, x_2, \dots, x_k)$ then $x_k\leq 0.6$, $x_{k-1}\leq 0.3$ and $x_{k-2}<0.3$.
$X$ cannot  be an infimum for $A$ and $B$ because, for each common bound, there is another lower bound, greater than the latter,  
$Y$, such that $Y^*=(x_1, x_2, \dots, \frac{x_{k-2}+x_{k-1}}{2}, x_{k-1}, x_k)$.

By duality, the optimistic order also cannot induce a lattice structure.
\end{example}

Another example of a \(\beta\)-normalization induced order is the one introduced by Garmendia et al.~in \cite[Proposition~2.6]{Garmendia17}, which extends the size of the compared sets through a least common multiple process. As noted in \cite[Corollary~3.7]{Garmendia17}, this order is not total and does not induce a lattice structure on THFEs.

An alternative approach to comparing THFEs involves truncating the set with more elements via  the so-called \(\alpha\)-normalization process ~\cite{Bedregal14}. Although such normalizations often fail to preserve transitivity~\cite{contra}, several orders have been constructed based on this idea, as those in~\cite[Definition~2.4]{Zhang2015}, \cite[Definition~3.3]{Zhang16}, and \cite[Remark~1]{Bedregal14}.  However, none of these orders endows a lattice structure on \( F^*([0,1]) \)~\cite[Theorem~3.8]{Xu2019}. In contrast, the orders presented in the following subsection are shown to endow a lattice structure on \( F^*([0,1]) \).

\subsection{Lattice orders on hesitant fuzzy elements}

Most orders on HFEs are focused solely on THFEs, becoming unfunctional when extended to sets that incorporate, for instance, intervals or infinitesimal sequences. A significant contribution in this direction is given in \cite{JMNS22}, where the authors introduce three new partial orders: the right and left orders, which, respectively, preserve the meet and join operations proposed by Torra; and the symmetric order, which endows the class of all nonempty subsets of \([0,1]\) with a lattice structure. This latter order do not only extends Zadeh’s classical operations on FSs and IVFSs but also solves the open problem posed in \cite[Challenge 5.3]{Rodriguez2015} and \cite[Remark 2]{Bustince2016},  concerning the existence of a lattice structure on SVFSs consistent with Zadeh’s operations. For completeness, we recall the basic notions and results about these orders.


Given two subsets \( X,Y \in P^*([0,1]) \), we say that \( X < Y \) if, for all \( x \in X \) and all \( y \in Y \), \( x < y \). Observe that, if \( X = \emptyset \) or \( Y = \emptyset \), then trivially \( X < Y \). For simplicity, we shall denote by $A^+$ and $A^-$ the supremum and the infimum, respectively, of a subset $A$ in $[0,1]$. 
The relations \( \leq^r \) and \( \leq^l \) on \( P^*([0,1]) \), introduced in \cite{JMNS22}, are given as follows. For two subsets \( A, B \in P^*([0,1]) \), we say that
\begin{enumerate}[$i)$]
\item $A \leq^r B \text{ if and only if } A^+ \leq B^+ \text{ and } A < B \setminus A$, 
\item $A \leq^l B \text{ if and only if } A^- \leq B^- \text{ and } A \setminus B < B$.
\end{enumerate}
If we denote by $P^+([0,1])$ and $P^-([0,1])$ the classes of subsets closed to the right and closed to the left in $[0,1]$, respectively, then \( (P^+([0,1]), \leq^r) \) and \( (P^-([0,1]), \leq^l) \) are lattices, see \cite{JMNS22} and \cite{Lobillo18}. Observe that the pairs \( (F^*([0,1]), \leq^l)\) and \( (F^*([0,1]), \leq^r)\) inherit a lattice structure.

\begin{example}\label{ex5} Observe that both orders on $F^*([0,1])$ are not the same. For instance, consider the finite subsets of [0,1], \( A = \{0.1, 0.2, 0.4, 0.5\} \) and \( B = \{0.2, 0.5, 0.9\} \). Then $A \leq^r B,$ but $A \not\leq^l B.$
\end{example}

For any two subsets \( A, B \in F^*([0,1]) \), we may define the order relation \( \leq^0 \) given by \( A \leq^0 B \) if, and only if, both \( A \leq^l B \) and \( A \leq^r B \). This order is called the symmetric order and can be extended to an order on $P^*([0,1])$.

\begin{definition}\cite[Definition 23]{JMNS22}
Given two subsets \( A, B \in P^*([0,1]) \), we say that \( A \leq^0 B \) if and only if     \( A < B \setminus A \) and \( A \setminus B < B \).
\end{definition}
The relation \( \leq^0 \) is a partial order on \( P^*([0,1]) \) with minimum element \( 0 = \{0\} \) and maximum element \( 1 = \{1\} \).

\begin{example}
Consider the subsets \( A=\{0.2\} \cup [0.3,0.6) \) and \( B=(0.4, 0.7] \cup \{0.8\} \). Then we may see  that \( A < B \setminus A = [0.6,0.7] \cup \{0.8\} \) and \( A \setminus B = \{0.2\} \cup [0.3,0.4] < B \), so \( A \leq^0 B \).
\end{example}

\begin{lemma}\label{L6}
For any  \( A, B \in P^*([0,1]) \), the following statements are equivalent:
\begin{enumerate}[$i)$]
    \item \( A \leq^0 B \), 
    \item  \( A \cap B = \emptyset \) and \( A < B \), or \( A \setminus B < A \cap B < B \setminus A \).
\end{enumerate}
\end{lemma}

\begin{proof}
First, suppose that \( A \leq^0 B \). If \( A \cap B = \emptyset \), then \( A < B \). On the other hand, suppose that \( A \cap B \neq \emptyset \). If \( A \backslash B = \emptyset \), then \( A \backslash B < A \cap B \). If \( A \backslash B \neq \emptyset \), let \( a \) be an element of \( A \backslash B \) and \( x \) an element of \( A \cap B \). It is clear that \( x \in B \). Since \( A \backslash B < B \), it follows that \( a < x \). Moreover, if \( B \backslash A = \emptyset \), it is trivially verified that \( A \cap B < B \backslash A \). If \( B \backslash A \neq \emptyset \), let \( y \) be an element of \( A \cap B \) and \( b \) an element of \( B \backslash A \). It is obvious that \( y \in A \). Since \( A < B \backslash A \), it follows that \( y < b \).

On the other hand, if \( A \cap B = \emptyset \) and \( A < B \), then  \( A < B \backslash A \) and \( A \backslash B < B \). Otherwise, suppose that \( A \cap B \neq \emptyset \) and \( A \backslash B < A \cap B < B \backslash A \). Let us see that \( A < B \backslash A \). The proof for \( A \backslash B < B \) is analogous. If \( B \backslash A = \emptyset \), then \( A < B \backslash A \). If \( B \backslash A \neq \emptyset \), let \( a \in A \) and \( b \in B \backslash A \). If \( a \in B \), then \( a \in A \cap B \) and, by hypothesis, \( a < b \). If \( a \notin B \), then \( a \in A \backslash B \). Since $A\cap B \not = \emptyset$, there exists some $c\in A\cap B$, with $a<c<b$. So the result follows. 
\end{proof}

As proved in \cite{JMNS22}, \( (P^*([0,1]), \leq^0) \) is a lattice. Actually, \( (F^*([0,1]), \leq^0) \) is a sublattice. The relationships among all the orders on THFE's described to this point are given in the following proposition.

\begin{proposition}\label{relOpt}
Let \( A, B \in F^*([0,1]) \). Then:
\begin{enumerate}[$i)$]
    \item \( A \leq^0 B \) if and only if \( A \leq^r B \) and \( A \leq^l B \).
    \item If \( A \leq^r B \), then \( A \leq_{\mathrm{opt}} B \).
    \item If \( A \leq^l B \), then \( A \leq_{\mathrm{pes}} B \).
    \item If \( A \leq^0 B \), then \( A \leq_{\mathrm{opt}} B \) and \( A \leq_{\mathrm{pes}} B \).
\end{enumerate}
\end{proposition}

\begin{proof}
$i)$ is proved in \cite{JMNS22}. 
Let us prove $ii)$. Assume \( A \leq^r B \). By definition, we know that \( A < B \setminus A \). Let \( r = \#A, s = \#B \), and set \( n = \max\{r, s\} \). Consider the sequences
\[
A^{(n)} = (a_1, a_2, \dots, a_n) 
\quad \text{and} \quad
B^{(n)} = (b_1, b_2, \dots, b_n),
\]
constructed as in the definition of the optimistic order. We will prove by induction that \(a_i \leq b_i\) for each \(i \in \{1, \dots, n\}\). Firstly, compare \(a_1\) and \(b_1\).  
If \(b_1 \in A\), then there exists some \(k\) with \(b_1 = a_k \geq a_1\). Hence \(a_1 \leq b_1\). Contrary to that, if \(b_1 \notin A\), then \(a_1 \in A\) and \(b_1 \in B \setminus A\). Since \(A < B \setminus A\), we get \(a_1 < b_1\).
Thus, \(a_1 \leq b_1\) in all cases.

Let us now suppose \(a_j \leq b_j\) for \(j < i\) and we will show \(a_i \leq b_i\).  
  If \(b_i \notin A\), then \(b_i \in B \setminus A\). Again \(a_i \in A\) forces \(a_i < b_i\), since \(A < B \setminus A\).
 Otherwise, if \(b_i \in A\), then \(b_i = a_k\) for some \(k\). If \(k \geq i\), we have \(a_i \leq a_k = b_i\). If \(k < i\), then by the inductive hypothesis, \(a_k = b_i \leq b_k\), but also \(b_k \leq b_i\), by the way we have repeated elements in \(B^{(n)}\). Thus \(b_k = b_i\). Consequently, \(b_i = b_n = B^+\), the maximum of \(B\). Hence \(a_i \leq A^+ \leq B^+ = b_i\).  This shows that \(A^{(n)} \leq_{\mathrm{prod}} B^{(n)}\). Therefore, \(A \leq_{\mathrm{opt}} B\).

$iii)$  follows from $ii)$ using that \( A \leq^l B \) if and only if \( 1-A \leq^r 1-B \) \cite[Proposition 15]{JMNS22}, and  \( A \leq_{\mathrm{pes}} B \) if and only if \( 1-A \leq_{\mathrm{opt}} 1-B \). 

$iv)$ follows from the previous statements.
\end{proof}

\section{Scores relative to an order}

This section addresses the development of scores for HFEs, emphasizing the importance of associating each score to a suitable order structure.  The foundational approach, introduced by Xia and Xu in \cite{XuXia11}, defines scores on \( F^*([0,1]) \) via the arithmetic mean, thereby directly linking the concept to a “mean value” interpretation. Later, Farhadinia \cite[Section 2]{Farhadinia14} expanded this idea by proposing scores on \( F^*([0,1]) \) as monotonic mappings \(s\) with respect to the list order, which satisfy the boundary conditions \(s(\{0\})=0\) and \(s(\{1\})=1\). This generalized setting enabled the inclusion of other classical aggregation operators, such as the geometric mean or the maximum. In \cite[Definition 2]{Alcantud2016} Alcantud el al. redefined scores as mapping, \(s,\) on \(F^*([0,1])\) satisfying that \(s(h)=1\) if and only if \(h=\{1\}.\) As highlighted in \cite{Alcantud2016}, this definition eliminates the ``unnecessary boundary conditions'' of the previous definition. More recently, Alcantud et al.~\cite{Alcantud} reformulated again the notion of score to encompass infinite subsets of \([0,1].\) 

\begin{definition} \label{defAlc} \cite[Definition 5]{Alcantud}
Given a family \( G \subseteq P([0,1]) \), a score on \( G \) is a map \( s: G \longrightarrow [0,1] \) with the following properties:
\begin{enumerate}[$i)$]
    \item \( s(\emptyset) = 0 \), whenever \( \emptyset \in G \);
    \item Boundedness. For all \( E \in G \), it holds that \( E^- \leq s(E) \leq E^+ \).
\end{enumerate}
\end{definition}

The boundedness condition implies that \( s(\{a\}) = a \) for any \( a \in [0,1] \), a property known as \emph{compatibility}, which ties the concept of a score to that of a “mean value” by requiring \( s(A) \in [A^-, A^+] \) for any \( A \subset [0,1] \).  Although boundedness ensures that the score reflects a membership degree within the support of the set, it may limit the capacity of designing score functions  in some decision-making scenarios. On the other hand, omitting the first condition in Definition~\ref{defAlc} renders the concept equivalent to evaluating scores over \( P^*([0,1]) \).

In this context, we propose the following definition, which relies solely on the explicit association of the score with a given order.

\begin{definition}\label{refScore}
Given a family \( G \subseteq P^*([0,1]) \) and an order \( \preccurlyeq \) on \( G \), a function \( s: G \rightarrow [0,1] \) is a \( \preccurlyeq \)-score (respectively, a strong \( \preccurlyeq \)-score) on \( G \) if, for all \( A, B \in G \), \( A \preccurlyeq B \) implies \( s(A) \leq s(B) \) (respectively, \( A \prec B \) implies \( s(A) < s(B) \)).
\end{definition}

This formulation integrates the idea that a score may be defined in association with a specific order \( \preccurlyeq \). Although such an assumption has been made when studying other hesitant fuzzy operators, such as typical hesitant fuzzy negations~\cite[Definition~7]{Santos2014} or typical hesitant triangular norms~\cite[Definition~12]{SantosBredegal} with respect to the pessimistic order, to the best of our knowledge, the connection with scores has not yet been properly addressed.

\begin{remark}Observe that, from Definition~\ref{refScore}, if \( G = F^*([0,1]) \), a score \( s \) in the sense of Farhadinia~\cite{Farhadinia14} qualifies as a \( \leq_{\text{list}} \)-score satisfying \( s(\{0\}) = 0 \) and \( s(\{1\}) = 1 \). It is worth noting that this condition imposes no restriction on the evaluation of sets with different cardinalities, thus allowing for diverse scoring behaviors even within the same ordering framework. This flexibility is preserved in our definition, which adheres to the original philosophy behind hesitant fuzzy scores.
\end{remark}

In the following proposition, we illustrate Definition \ref{refScore} with concrete examples of some classical aggregation functions which are \( \leq^0 \)-scores.

\begin{proposition}
The arithmetic mean, the geometric mean, the minimum, and the maximum are \( \leq^0 \)-scores on $F^*([0,1])$.
\end{proposition}

\begin{proof}
Let \( m \) denote the average operator, and suppose two different subsets \( A \) and \( B \) such that \( A \leq^0 B \). By Lemma~\ref{L6}, we know that either \( A \cap B = \emptyset \) and \( A < B \), and hence \( m(A) < m(B) \), or \( A \cap B \neq \emptyset \) and \( A \setminus B < A \cap B < B \setminus A \). Let \( t = \#(A \cap B) \), \( r = \#A \), and \( s = \#B \). Then:
\begin{itemize}
    \item If \( A \subseteq B \), then \( A \cap B = A < B \setminus A \), so \( m(A) < m(B \setminus A) \), and
    \[
    m(B) = \frac{(s - t) \cdot m(B \setminus A) + t \cdot m(A)}{s} > m(A).
    \]
    \item Similarly, if \( B \subseteq A \), we deduce \( m(A) < m(B) \).
    \item If \( A \not\subseteq B \) and \( B \not\subseteq A \),
    \[
    m(A) = \frac{(r - t) \cdot m(A \setminus B) + t \cdot m(A \cap B)}{r} < m(A \cap B),
    \]
    and similarly \( m(B) > m(A \cap B) \), hence \( m(A) < m(B) \).
\end{itemize}
An analogous argument applies for the geometric mean. For the minimum and maximum, the result follows  from Proposition~\ref{relOpt}, $i)$.
\end{proof}

Not all scores that are compatible with the list order are also compatible with the symmetric order. A straightforward counterexample is the product.

\begin{example}
The product $p$ is clearly a score in the sense of Farhadinia. Nevertheless, consider \( A = \{0.1\} \) and \( B = \{0.2, 0.3\} \). Then \( A \leq^0 B \), but the product yields \( p(B) = 0.06 < 0.1 = p(A) \), so it is not a \( \leq^0 \)-score.
\end{example}

An advantage of using scores relative to the symmetric order is that they naturally align with a well-defined lattice structure. As an consequence of this structure, we may proof the following result.

\begin{proposition}
If \( s \) is a \( \leq^0 \)-score, then \( s(A \wedge_0 B) \leq \min\{s(A), s(B)\} \) and \( \max\{s(A), s(B)\} \leq s(A \vee_0 B) \).
\end{proposition}

\begin{proof}
Since \( A \wedge_0 B \leq^0 A \) and \( A \wedge_0 B \leq^0 B \), it follows that \( s(A \wedge_0 B) \leq s(A) \) and \( s(A \wedge_0 B) \leq s(B) \), hence \( s(A \wedge_0 B) \leq \min\{s(A), s(B)\} \). The second inequality follows analogously.
\end{proof}

We must point out that if \( \leq_2 \) is a refinement of \( \leq_1 \), then any \( \leq_2 \)-score is also a \( \leq_1 \)-score. This observation, together with Proposition~\ref{relOpt}, leads to the following result.

\begin{corollary}\label{clasifScores}
Given a score \( s:F^*([0,1]) \rightarrow [0,1] \), the following implications hold:
\begin{enumerate}
    \item If \( s \) is a \( \leq_{\text{opt}} \)-score, it is also a \( \leq_{\text{list}} \)-score.
    \item If \( s \) is a \( \leq_{\text{pes}} \)-score, it is also a \( \leq_{\text{list}} \)-score.
    \item If \( s \) is a \( \leq_{\text{opt}} \)-score, it is also a \( \leq^r \)-score.
    \item If \( s \) is a \( \leq_{\text{pes}} \)-score, it is also a \( \leq^l \)-score.
    \item If \( s \) is a \( \leq^r \)-score, it is also a \( \leq^0 \)-score.
    \item If \( s \) is a \( \leq^l \)-score, it is also a \( \leq^0 \)-score.
\end{enumerate}
\end{corollary}

\begin{proof}
These implications follow directly from the hierarchy of order refinements established in Section~\ref{StateOrders}.
\end{proof}

\section{Properties of $\leq^0$-scores}

Alcantud et al. \cite{Alcantud} provide a landmark contribution to the study of scores for HFEs, establishing a rigorous framework that addresses both their desirable properties and inherent limitations. Their approach formalizes the definition of scores, focusing on their internal coherence. Building on this groundwork, $\leq^0$-scores naturally align with Alcantud et al.’s normative framework while introducing meaningful refinements. These scores not only satisfy critical properties such as strong monotonicity with respect to unions ([SMU]) and the Gärdenfors property ([G]), but also enhance concrete interpretability and flexibility in modeling preferences within hesitant fuzzy environments. In this sense, our contribution reinforces Alcantud et al.'s insights, offering a refined perspective on scores for HFEs.


In \cite{Alcantud}, it is established that if a score satisfies the property of boundedness, it also satisfies the property of compatibility, but the converse is not true in general. Nevertheless, this holds when working with $\leq^0$-scores.

\begin{proposition}
Let $s : P^*([0,1]) \to [0,1]$ be a $\leq^0$-score. Then $s$ satisfies compatibility if and only if it satisfies boundedness.
\end{proposition}

\begin{proof}  From the definition of symmetric order, we may state that $\{A^-\} \leq^0 A \leq^0 \{A^+\}$, and then, since $s$ is a  $\leq^0$-score, $s(\{A^-\})\leq s(A) \leq s(\{A^+\})$. By the compatibility property, $s(\{A^-\})=A^-$ and $s(\{A^+\})=A^+$, so boundedness occurs.
\end{proof}

In \cite{Alcantud}, one of the properties, regarded as natural for a score, is the property [SMU]. In the following definition, we extend it to any arbitrary mapping.

\begin{definition} \cite[Definition 9]{Alcantud} \label{def:SMU}
Let \( G \subseteq P^*([0,1]) \). A map \( s:G \to [0,1] \) is strongly  monotonic with respect to unions [SMU] if, for any \( X, Y\in G\) such that \(X \cup Y \in G \) and \( X < Y \), it holds that \( s(X) < s(X \cup Y) < s(Y) \).
\end{definition}

\begin{remark}
Using the same argument as in \cite[Lemma 1]{Alcantud}, we can prove that, if $G$ contains all the subintervals in [0,1], then there cannot exist a score in $G$ verifying [SMU]. Indeed, for each $a\in ]0,1]$, applying [SMU] to $[0,a[<\{a\}$, then $s([0,a[)<s([0,a])$ and there must exist a rational $q_a$ such that
$s([0,a[)<q_a<s([0,a])$. Given $a, b\in ]0,1]$ with $a<b$, applying [SMU] to $[0,a]<[0,b[$, we obtain that $s([0,a])<s([0,b[)$ and then $q_a<q_b$ so we have an injective map from $\mathbb{R}$ in $\mathbb{Q}$ in contradiction with the known fact of having $\mathbb{R}$ greather cardinal than $\mathbb{Q}$. 
\end{remark}

This moves us to define a weaker versión of this property.
\begin{definition} \label{def:WMU}
Let \( G \subseteq P^*([0,1]) \). A map \( s:G \to [0,1] \) is weakly  monotonic with respect to unions [WMU] if, for any \( X, Y\in G\) such that \(X \cup Y \in G \) and \( X < Y \), it holds that \( s(X) \leq s(X \cup Y) \leq s(Y) \).
\end{definition}


\begin{lemma}\label{lem:SMU}
Let $X,Y\in P^*([0,1])$ such that $X\cap Y = \emptyset$. The following are equivalent:
\begin{enumerate}[$i)$]
\item $X<Y$.
\item $X<^0 Y$.
\item $X<^0 X\cup Y$.
\item $X\cup Y <^0 Y$.
\end{enumerate}
\end{lemma}
\begin{proof}
By Lemma \ref{L6}, $i)$ and $ii)$ are equivalent. Now,  since $X\cap Y = \emptyset$,   \( X \setminus (X\cup Y) = \emptyset\), $X\cap (X\cup Y)=X$ and $(X\cup Y)\setminus X=Y$. Therefore  $i)$ is equivalent to $iii)$. Similarly, \( X= (X\cup Y) \setminus Y\), $Y= (X\cup Y)\cap Y$ and $\emptyset = Y\setminus (X\cup Y)$, so $i)$ is equivalent to $iv)$.
\end{proof}

The relation between $\leq^0$-scores and the [SMU] property is given as follows.

\begin{theorem}\label{prop:SMU}
Let $G\subseteq P^*([0,1])$ be a class closed by unions and differences, and 
 \( s:G \to [0,1] \) a map.  Then:
\begin{enumerate}[$i)$]
\item
$s$ satisfies the \emph{[WMU]} property if and only if it is a \( \leq^0 \)-score.
\item
$s$ satisfies the \emph{[SMU]} property if and only if it is a strong \( \leq^0 \)-score.
\end{enumerate} 
\end{theorem}

\begin{proof}
$i)$  Suppose that $s$ is [WMU].   Let \( A ,B \in G\) such that \( A <^0 B \). If \( A \cap B = \emptyset \), then \( A < B \) and, since \( s \) is [WMU], it follows that \( s(A) \leq s(B) \). Otherwise,  if \( A \cap B \neq \emptyset \), we know that \( A \setminus B < A \cap B < B \setminus A \). We distinguish three cases. Firstly, if \( A \subset B \), since \( s \) is [WMU], from the inequality \( A = A \cap B < B \setminus A \), it follows that
    \[
    s(A) \leq s((B \setminus A) \cup A) = s(B) \leq s(B \setminus A)
    \]
    and the first inequality proves the condition. Secondly, if \( B \subset A \), since \( s \) is [WMU], from the inequality \( A \setminus B < A \cap B = B \), it follows that
    \[
    s(A \setminus B) \leq s((A \setminus B) \cup B) = s(A) \leq s(A \cap B) = s(B)
    \]
    and the second inequality proves the condition. Finally, if \( A \not\subset B \) and \( B \not\subset A \), let us denote
    \[
    X = A \setminus B \neq \emptyset, \, Y = A \cap B \neq \emptyset, \text{ and } Z = B \setminus A \neq \emptyset.
    \]
    Note that \( X, Y, Z \in G \) with \( X < Y < Z \), \( X \cup Y = A\) and  \(Y \cup Z = B \).
    Since \( X < Y \),
    $
    s(X) \leq s(X \cup Y) = s(A) \leq s(Y),
    $
    and, since \( Y < Z \),
    $
    s(Y) \leq s(Y \cup Z) = s(B) \leq s(Z),
    $
    from which we obtain \( s(A) \leq s(B) \).

    Conversely, suppose that \( s \) is a \( \leq^0 \)-\textit{score} and \( X < Y \).  Therefore $X\cap Y=\emptyset$ and, by Lemma \ref{lem:SMU}, $X<^0X\cup Y <^0Y$. Thus $s(X) \leq s(X\cup Y) \leq s(Y)$.
    
The proof of $ii)$ is entirely similar.    
\end{proof}

\begin{remark}\label{remarkF}
Observe that Theorem \ref{prop:SMU} is valid, in particular, for maps  $s:F^*([0,1])\to [0,1]$. Therefore it satisfies the [SMU] property if and only if it is a strong \( \leq^0 \)-score on $F^*([0,1])$.
\end{remark}

Another properties related to symmetric scores are the following.

\begin{definition} \cite[Definition 12]{Alcantud} \label{def:G}
Given \( G \subseteq P^*([0,1]) \), a map \( s \) on \( G \) satisfies the Gärdenfors property [G] if, for every \( A \in P^*([0,1]) \) and any \( x \not\in A \) with \( A, A \cup \{x\} \in G \), the following conditions hold:
\begin{enumerate}[$i)$]
\item $x < A^- \Rightarrow s(A \cup \{x\}) < s(A)$,
\item $A^+ < x \Rightarrow s(A) < s(A \cup \{x\})$.
\end{enumerate}
\end{definition}

We introduce here a weaker version of the property.

\begin{definition} \label{def:WG}
Given \( G \subseteq P^*([0,1]) \), a score \( s \) on \( G \) satisfies the weak Gärdenfors property [WG] if for every \( A \in P^*([0,1]) \) and any \( x \not\in A \) with \( A, A \cup \{x\} \in G \), the following conditions hold:
\begin{enumerate}[$i)$]
\item $x < A^- \Rightarrow s(A \cup \{x\}) \leq s(A)$,
\item $A^+ < x \Rightarrow s(A) \leq s(A \cup \{x\})$.
\end{enumerate}
\end{definition}

\begin{remark}\label{remarkG}
Observe that, if $x < A^-$, then $\{x\} < A$, and, if $A^+ < x$, then $A<\{x\}$. Thus, it is clear  that [WMU] implies [WG], and [SMU] implies [G]. 
\end{remark}

Whenever \( G \) is the set of all THFEs, the following statement hold.

\begin{proposition}\label{prop:WMU}
Let \( s:F^*([0,1]) \rightarrow [0,1] \) be a map. Then it is satisfied that:
\begin{enumerate}[$i)$]
\item 
The following conditions are equivalent:
\begin{enumerate}
    \item \( s \) is a \( \leq^0 \)-score.
    \item \( s \) verifies \emph{[WMU]}
    \item  \( s \) verifies \emph{[WG]}
\end{enumerate}
\item
The following conditions are equivalent:
\begin{enumerate}
    \item \( s \) is a strong \( \leq^0 \)-score.
    \item \( s \) verifies \emph{[SMU]}
    \item  \( s \) verifies \emph{[G]}
\end{enumerate}
\end{enumerate}
\end{proposition}

\begin{proof}
$i$) Taking into account remarks \ref{remarkF} and \ref{remarkG}, it only remains to prove that $(c)$ implies $(b)$.  
Let $X<Y$ and assume  that $X=\{x_1, \dots, x_r\}$ and $Y=\{y_1, \dots, y_s\}$ with $x_1<x_2<\dots <x_r<y_1<\dots <y_s$. Using [G] repeatedly, for every $i=1, \dots, r-1$,  
we have $ x_i < \{x_{i+1},\dots, x_r,y_1, \dots, y_s\}^-$ and we obtain that 
$s(\{x_i,\dots, x_r,y_1, \dots, y_s\})\leq s(\{x_{i+1},\dots, x_r,y_1, \dots, y_s\})$, and then $s(X\cup Y)\leq s(Y)$.

Similarly, for every $j=1,\dots, s$, $\{x_1,\dots, x_r, y_1, \dots, y_j\}^+<y_{j+1} $ and we obtain 
$s(\{x_1,\dots, x_r, y_1, \dots, y_j\})\leq s(\{x_1,\dots, x_r, y_1, \dots, y_j, y_{j+1}\})$. Therefore $s(X)\leq s(X\cup Y)$

  The proof of $ii)$ is similar to the previous one.
\end{proof}

Another important case is that of the interval-valued fuzzy elements. 
Let us denote by $I^c([0,1])$ the class of all closed subintervals in $[0,1]$.

\begin{definition} \cite[Definition 15]{Alcantud} \label{def:EM}
Let $G\subseteq P^*([0,1])$ be a class that contains $I^c([0,1])$, a map $s:G\to [0,1]$  is extremes monotonic [EM] when it satisﬁes the
following two conditions:
\begin{enumerate}
\item[1.] [EM1]. $0 \leq b < b' \leq 1$ implies $s([a, b]) < s([a, b' ])$ for each $a \in [0, b]$,
\item[2.] [EM2]. $0 \leq a < a' \leq 1$ implies $s([a, b]) < s([a' , b])$ for each $b \in [a' , 1]$.
\end{enumerate}
\end{definition}

\begin{proposition}\label{prop:WG}
Let \( s:I^c([0,1]) \rightarrow [0,1] \) be a map. 
 Then the following statements are satisfied.
 \begin{enumerate}[$i)$]
\item 
\( s \) is a \( \leq^0 \)-score if and only if \( s \) verifies that $a\leq c$ and $b\leq d$ implies $s([a,b])\leq s([c,d])$.
\item
\( s \) is a strong \( \leq^0 \)-score if and only if \( s \) verifies [EM].
\end{enumerate}
\end{proposition}

\begin{proof}
$i)$ follows from \cite[Theorem 44]{JMNS22}, that is, given two closed intervals $[a,b]$ and $[b,c]$, $[a,b]\leq^0 [c,d]$ if and only if $a\leq c$ and $b\leq d$.  $ii)$ follows from $i)$ and Definition \ref{def:EM}.

\end{proof}

\section{Dominance functions}

In this section we introduce a class of functions named \emph{dominance functions} (DFs) relative to some order. Unlike traditional fuzzy scores that focus on direct evaluations, we try to incorporate one or more baseline requirements when evaluating each HFE. By specifying an “acceptable” or “ideal” set \(B\), dominance scores quantify how far a set \(A\) exceeds or fails to meet the expectations encoded in \(B\). This design is particularly useful in decision contexts that involve regulatory minimums or any application where a baseline constraint must be factored. In such scenarios, a DF provides a measure of whether a set is sufficiently above (dominates), below (dominated by) or indifferent with respect to the defined standard. The conceptual basis of this formulation lies in the classical conception of \emph{preference or dominance relations} established by Davis and Hinich~\cite{Davis1972} and Nitzan and Rubinstein~\cite{Nitzan1976}.

\begin{definition}[Dominance function relative to an order]
Let $G\subseteq  P^*([0,1]) \) and $\preccurlyeq$ an order on $G$. A dominance function on $G$ relative to $\preccurlyeq$ is a map 
\[
D : G \times G \to [0,1]
\]
that satisfies the following properties:
\begin{enumerate}[$i)$]
    \item \( D(X,X) = 0.5 \) for any $X\in G$, and,
    \item for each $X,Y,Z\in G$, if \( Y \preccurlyeq Z \), then \( D(X,Y) \leq D(X,Z) \).
\end{enumerate}
Let $X ,Y,Z\in G$, we say that $Y$ dominates $Z$ in relation to $X$ if $D(X,Y)>D(X,Z)$. We say that $Y$ and $Z$ are indifferent in relation to $X$ if $D(X,Y) = D(X,Z)$.
\end{definition}

Observe that to give a dominance function is equivalent to give a family $\{D_X\}_{X\in G}$ of $\preccurlyeq$-scores satisfying that $D_X(X)=0.5$ for all $X\in G$. The underlying idea of dominance function is to use a control  set $X$ in such a way the higher the dominance value, the better the evaluated set is in relation to this acceptable set \( X \), in the sense that it further exceeds the minimum acceptable expectations.

Next, we show two specific examples of dominance functions on THFEs.  Given two elements $x,y \in [0,1]$, we consider 
\begin{equation}     s(x,y)     \;=\;     \begin{cases}     1, & \text{if } x < y,\\     0.5, & \text{if } x = y,\\     0, & \text{if } x > y,    \end{cases}  
\end{equation} 
and \begin{equation}r(x,y) = \frac{1}{2}(y-x+1).\end{equation}

\begin{definition}
\label{def:DDS}
Let \(A, B \in F^*([0,1])\), we define the functions 
$S,R: F^*([0,1]) \times F^*([0,1]) \to [0,1]$ 
as  $$S(A,B) =  \frac{1}{\#A \cdot   \#B}\sum_{a\in A}\sum_{b\in B} s(a,b).$$
$$R(A,B) =\frac{1}{\#A \cdot   \#B}\sum_{a\in A}\sum_{b\in B} r(a,b).$$
\end{definition}

Alternatively,  if  \(A = \{a_1, \dots, a_n\}\) with $a_1<a_2<\cdots < a_n$ and \(B = \{b_1, \dots, b_m\}\) with $b_1<b_2<\cdots < b_m$,  we may construct the $n\times m$ matrices $S_{A,B} = \left (s(a_i,b_j) \right )_{i,j}$  and $R_{A,B} = \left (r(a_i,b_j) \right )_{i,j}$. 

\begin{theorem}
The functions $R$ and $S$ are dominance functions on $F^*([0,1])$ relative to the list order and the symmetric order.
\end{theorem}
\begin{proof}
It is easy to check that $S(X,X)=R(X,X) = 0.5$ for all $X\in F^*([0,1])$. Let us prove that, for any $A,B,C\in F^*([0,1])$, $S(A,B)\leq S(A,C)$ whenever $B\leq_{list} C$. Assume first that \(\#B = \#C = n\). Since \(b_i \leq c_i\) for each \(i\), we compare the elements of the matrices $S_{A,B}$ and $S_{A,C}$. Let \(a \in A\):
\begin{itemize}
\item  If \(a < b_i\), since \(b_i \leq c_i\), \(a < c_i\).  Then $1= s(a,b_i) = s(a,c_i)$.
\item If \(a = b_i\), then \(a \leq c_i\).  
  So $0.5 = s(a,b_i) \leq s(a,c_i)$.
\item If \(a > b_i\), then \(0 = s(a,b_i) \leq s(a,c_i) \).  
\end{itemize}
Consequently, $S(A,B) \leq S(A,C)$.

Suppose now that \(\#B \neq \#C\), therefore \(B = \{0\}\) or \(C = \{1\}\). Suppose that  $B=\{0\}$. If $0\notin A$, then $S(A,B)=0$, so  it holds. Otherwise, if $0\in A$, $\# A \cdot S(A,B) = 0.5$. Observe that  $s(0,c)=1$ for all $c\in C$, unless for the possible case $c=0$, whose value in 0.5. Then $ \sum_{c\in C} s(0,c) \geq 0.5 + \# C -1 = \#C - 0.5$. Therefore $$\# A \cdot S(A,C) = \frac{\sum_{c\in C} s(0,c)+ \Sigma}{\#C} \geq \frac{\# C -0.5 + \Sigma}{\# C} \geq 0.5,$$
where $\Sigma$ is the sum of the values $s(a,c)$ for any elements $0\not = a \in A$ and $c\in C$. So, the inequality holds. The case $C=\{1\}$ may be proved similarly. Thus $S$ is a dominance function relative to the list order.

With respect to the list order, suppose that $B\leq_{list} C$.
 Assume first that \(\#B = \#C = n\). Since \(b_i \leq c_i\) for each \(i\), then $$r(a_j,b_i)=\frac{1}{2}(b_i-a_j+1) \leq \frac{1}{2}(c_i-a_j+1)=r(a_j,c_i).$$
 Then, clearly, $R(A,B)\leq R(A,C)$.
 
 On the other hand, if $\#B \not = \# C$, hence, necessarily, $B=\{0\}$ or $C=\{1\}$. If $B=\{0\}$,
$$R(A,B)=\displaystyle \frac{1}{\#A} \sum_{a\in A} \frac{1}{2}(1-a).$$
But,
$$\begin{array}{rl}
R(A,C) & =\displaystyle \frac{1}{\#A\cdot  \# C} \sum_{a\in A, c\in C} \frac{1}{2} (c-a+1)  \\ & \geq \displaystyle \frac{1}{\#A \cdot \# C} \sum_{a\in A, c\in C} \frac{1}{2} (1-a) \\  & =\displaystyle \frac{1}{\#A} \sum_{a\in A} \frac{1}{2 }(1-a) \\ & = R(A,B).
\end{array}$$
 The case $C=\{1\}$ may be proved similarly.

Let us now prove that $S$ and $R$ are dominance functions relative to the symmetric order. By Proposition \ref{prop:WMU}, this is equivalent to prove that $S$ and $R$ are [WMU].  Let $A,B,C\in F^*([0,1])$ with $B< C$.  For all $a\in A$, all $b\in B$ and all $c\in C$, $s(a,b)\leq s(a,c)$ and $r(a,b)\leq r(a,c)$. So the mean of all numbers $s(a,b)$ is lower or equal than the mean of all numbers $s(a,c)$. That is, $S(A,B) \leq S(A,C)$. For the same reason, $R(A,B)\leq R(A,C)$. Now, if $r=\# B$ and $t = \# C$,
$$\begin{array}{rl}
s(A,B\cup C) & = \displaystyle \frac{1}{\#A \cdot (r+t)}\sum_{a\in A} \sum_{x\in B\cup C} s(a,x) \\
& \\
& = \displaystyle \frac{1}{r+ t} \left [ \left (\frac{1}{\#A} \sum_{a\in a} \sum_{b\in B} s(a,b) \right ) + \left (\frac{1}{\#A} \sum_{a\in a} \sum_{c\in C} s(a,c) \right ) \right ] \\
& \\
& = \displaystyle \frac{r \cdot s(A,B) + t \cdot  s(A,C)}{r+t}
\end{array}
$$
Then,
$$S(A,B) \leq  S(A,B\cup C)  =\frac{r \cdot S(A,B) + t \cdot S(A,C)}{r+ t}
 \leq S(A,C).$$
 
 This reasoning is also valid for the map $R$.
\end{proof}
 We call $S$ the discrete dominance function (DDF), and $R$ the relative dominance function (RDF). In general, the choice of a particular function depends on the characteristics of the system. For instance, the DDF appears to be more appropriate for a winner-take-all electoral system, whilst the RFD seems more suited for a proportional representation system. 
 
 Next, we present a numerical example.

\begin{example}\label{ExDDS}
Given \( A=\{0.3, 0.4\} \), \( B=\{0.1, 0.2, 1\} \) and \( C=\{0.2, 0.4\} \).  We calculate the matrices  
     $$ S_{A,B}=\left (\begin{array}{ccc} 0 & 0 & 1\\ 0& 0&  1\end{array} \right)  \text{ and }
      S_{A,C}=\left (\begin{array}{cc} 0& 1\\ 0& 0.5\end{array}\right ).$$
Thus, the DDF values are \( S(A,B)=0.33 \), and \( S(A,C)=0.375 \).
Therefore, in this case, the ranking of the alternatives, relative to the ideal set \( A \), indicates that $C$ is preferable to $B$. 

Now we may compute the matrices
     $$ R_{A,B}=\left (\begin{array}{ccc} 0.4 & 0.45 & 0.85\\ 0.35& 0.4& 0.8\end{array} \right) \text{ and }
     R_{A,C}=\left (\begin{array}{cc} 0.45& 0.55\\ 0.4&  0.5 \end{array}\right )$$ 
and the RDF values are $R(A,B)=0.541$ and  $R(A,C)=0.47$.  Therefore, in this case, the set $B$ is preferable to $C$ with respect to the set $A$.
\end{example}

We state now that both the DDF and the RDF satisfy certain complementary properties, which are outlined below.

\begin{lemma}\label{lem:SBplusOneMinusSB}
For every \( A, B \in F^*([0,1]) \), it is satisfied 
\[
S(A,B) + S(B,A) = 1 \text{ and } S(1-A,1-B) = S(B,A),
\]
and 
\[
R(A,B) + R(B,A) = 1 \text{ and } R(1-A,1-B) = R(B,A).
\]
\end{lemma}
\begin{proof}
The properties follow straightforwardly from the facts $s(x,y)+s(y,x) = 1$, $r(x,y)+r(y,x) = 1$, $s(1-x,1-y) = s(y,x)$ and $r(1-x,1-y) = r(y,x)$ for all $x,y\in [0,1]$.
\end{proof}

\subsection{Construction of preference relations}

Advanced consistency analyses for fuzzy preference relations, as explored by Xu \cite{Xu2015}, provide critical information to improve the effectiveness of group decision making. In this context, dominance scores offer a useful approach to construct preference relations in group decision-making problems. 
Originally, Orlovsky \cite{Orlo93} defined the concept of a fuzzy preference relation as follows.

\begin{definition} \cite[Definition 11]{Chen2013} \cite{Orlo93} 
Let \( X = \{x_1, \dots, x_n\} \) be a set of discrete alternatives. A fuzzy preference relation \( R \) on \( X \) is represented by a matrix \( R = (r_{ij})_{n \times n}\) with coefficients in $[0,1]$, such that:
\begin{enumerate} [$i)$]
    \item \( r_{ij} \geq 0 \), 
    \item \( r_{ij} + r_{ji} = 1 \), 
    \item \( r_{ii} = 0.5 \) for all \( i, j = 1, \dots, n \), 
\end{enumerate} 
where \( r_{ij} \) denotes the preference degree of the alternative \( x_i \) over \( x_j \).
\end{definition}

In this structure, \( r_{ij} = 0.5 \) indicates indifference between \( x_i \) and \( x_j \), while \( r_{ij} > 0.5 \) suggests a preference for \( x_i \) over \( x_j \), and \( r_{ij} < 0.5 \) indicates a preference for \( x_j \) over \( x_i \). The larger the value of \( r_{ij} \), the stronger the preference for \( x_i \) over \( x_j \).

\begin{proposition}\label{thm:PrefRelation}
Let $X=\{A_1,\ldots , A_n\}$, where $A_i\in F^*([0,1])$. Then $S$ and $R$ determine preference relations on $X$ given by the matrices $\left (S(A_i,A_j) \right )_{i,j\in \{1,\ldots ,n\}}$  and $\left (R(A_i,A_j) \right )_{i,j\in \{1,\ldots ,n\}}$, respectively.
\end{proposition}

\begin{proof} It follows from Lemma \ref{lem:SBplusOneMinusSB}.
 \end{proof}

Let us consider an example of how a preference relation can be constructed using DSFs.

\begin{example}\label{prefRel} 
Suppose we have three alternatives with values $A_1 = \{0.7, 0.8\}$,  $A_2 = \{0.75, 0.8\}$ and  $A_3 = \{0.8, 0.9\}$.

Using the DDF, we construct a preference matrix 
\[ P  = \left (S(A_i,A_j) \right )_{i,j=1,2,3}=
\begin{pmatrix}
0.5 & 0.375 & 0.125 \\
0.625 & 0.5 & 0.125 \\
0.875 & 0.875 & 0.5
\end{pmatrix}
\]
Similarly, applying the RDF, we obtain a different preference matrix that incorporates distance-based evaluations between elements.
\[ Q  = \left (R(A_i,A_j) \right )_{i,j=1,2,3}=
\begin{pmatrix}
0.5 & 0.5125 & 0.55 \\
0.4875 & 0.5 & 0.4625 \\
0.45 & 0.5375 & 0.5
\end{pmatrix}
\]

\end{example}

\subsection{Example: application of dominance functions for project evaluation}

In this subsection, we present an illustrative model to demonstrate the application of dominance functions. The model simulates a decision-making context where several projects are evaluated relative to predefined baseline expectations across different criteria. The setting involves three evaluation criteria for each project:
\begin{itemize}
    \item C1: Innovation level,
    \item C2: Feasibility,
    \item C3: Sustainability.
\end{itemize}

A control group of expert evaluators establishes baseline THFEs for each criterion, representing minimal acceptable standards. Four projects are then evaluated based on expert assessments, where each evaluation is modeled as a THFE. The dominance of each project with respect to the baseline is assessed using both the DDF and the RDF. The baseline THFEs for each criterion are given by:
\begin{itemize}
    \item C1 (Innovation level baseline): $\{0.7, 0.8\}$,
    \item C2 (Feasibility baseline): $\{0.7\}$,
    \item C3 (Sustainability baseline): $\{0.5, 0.6\}$.
\end{itemize}

The evaluation THFEs for the four projects are as follows:

\begin{table}[!h]
\centering
\begin{tabular}{|c|c|c|c|}
\hline
\textbf{Project} & \textbf{C1} & \textbf{C2} & \textbf{C3} \\ \hline
P1 & $\{0.6, 0.7\}$ & $\{0.6, 0.7\}$ & $\{0.6\}$ \\ \hline
P2 & $\{0.5, 0.6\}$ & $\{0.5\}$ & $\{0.4, 0.5\}$ \\ \hline
P3 & $\{0.8\}$ & $\{0.6, 0.7\}$ & $\{0.5, 0.6\}$ \\ \hline
P4 & $\{0.5\}$ & $\{0.3, 0.4\}$ & $\{0.4\}$ \\ \hline
\end{tabular}
\caption{Project evaluation THFEs for each criterion.}
\label{tab:data-toy}
\end{table}

The dominance functions obtained for each project using the DDF and the RDF are summarized below in Tables \ref{tab:scoresddf} and \ref{tab:scoresrdf}, respectively:

\begin{table}[!h]
\centering
\begin{tabular}{|c|c|c|c|}
\hline
\textbf{Project} & \textbf{DDF (C1)} & \textbf{DDF (C2)} & \textbf{DDF (C3)} \\ \hline
P1 & 0.125 & 0.25 & 0.75 \\ \hline
P2 & 0 & 0 & 0.125 \\ \hline
P3 & 0.75 & 0.25 & 0.5 \\ \hline
P4 & 0 & 0 & 0 \\ \hline
\end{tabular}
\caption{DDF  by project and criterion.}
\label{tab:scoresddf}
\end{table}

\vspace{0.5cm}

\begin{table}[!h]
\centering
\begin{tabular}{|c|c|c|c|}
\hline
\textbf{Project} & \textbf{RDF (C1)} & \textbf{RDF (C2)} & \textbf{RDF (C3)} \\ \hline
P1 & 0.45 & 0.475 & 0.525 \\ \hline
P2 & 0.4 & 0.4 & 0.45 \\ \hline
P3 & 0.525 & 0.475 & 0.5 \\ \hline
P4 & 0.375 & 0.325 & 0.425 \\ \hline
\end{tabular}
\caption{RDF by project and criterion.}
\label{tab:scoresrdf}
\end{table}

\vspace{0.5cm}

The average values for each project are calculated separately for each dominance function values in order to rank the projects and derive a final decision. Nevertheless, any other aggregation function could be used. The corresponding rankings are summarized in Tables \ref{tab:rankings-rdf-toy} and \ref{tab:rankings-rdf-toy2}.

\begin{table}[!h]
\centering
\begin{tabular}{|c|c|c|}
\hline
\textbf{Rank (DDF)} & \textbf{Project} & \textbf{Average (DDF)} \\ \hline
1 & P3 & 0.5 \\ \hline
2 & P1 & 0.375 \\ \hline
3 & P2 & 0.0417 \\ \hline
4 & P4 & 0 \\ \hline
\end{tabular}
\caption{Project rankings based on average DDF.}
\label{tab:rankings-rdf-toy}
\end{table}

\begin{table}[!h]
\centering
\begin{tabular}{|c|c|c|}
\hline
\textbf{Rank (RDF)} & \textbf{Project} & \textbf{Average (RDF)} \\ \hline
1 & P3 & 0.5 \\ \hline
2 & P1 & 0.4833 \\ \hline
3 & P2 & 0.4167 \\ \hline
4 & P4 & 0.375 \\ \hline
\end{tabular}
\caption{Project rankings based on average RDF.}
\label{tab:rankings-rdf-toy2}
\end{table}

Observe that the DDF strictly penalizes any failure to meet the baseline values, resulting in sharper rankings that favor only those projects that fully satisfy the standards. In contrast, the RDF introduces a more continuous adjustment by considering the proximity to the baselines, leading to a more tolerant and nuanced ranking. The observed differences emphasize the importance of selecting the appropriate dominance function depending on the evaluation objectives.

\section{Conclusions} 

Several studies, see for instance \cite{Stateofart_HFS,Rodriguez2015}, highlight the importance of integrating hesitant fuzzy set theory into practical frameworks. This paper responds to that need by proposing an order-theoretic approach to scoring HFEs. Our framework expands the class of scores and clarifies their behavior through their explicit dependence on underlying order structures.
In this sense,  we have introduced the notion of \(\preccurlyeq\)-scores, formally defined with respect to arbitrary orders. This generalization moves beyond conventional list-order assumptions and accommodates a broader range of comparison criteria. By analyzing classical and alternative orders, we have shown how the order compatibility critically shapes score properties. Concretely, the symmetric order yields scores satisfying most of the essential normative properties introduced in \cite{Alcantud}. As an application we have introduced the dominance functions aiming to improve the expressive power of decision-making hesitant fuzzy models.

\section*{Statements and Declarations}

\bmhead{Conflict of Interest} 
The authors declare that they have no conflicts of interest.

\bmhead{Ethics Approval} This is a theoretical study that does not involve human participants, animal subjects, or sensitive data.

\bmhead{Funding} 
This work is partially supported by the “María de Maeztu” Excellence Unit IMAG, reference CEX2020-001105-M, funded by MCIN/AEI/10.13039/501100011033/.

\bmhead{Data Availability} 
Not applicable.

\bmhead{Author Contribution} 
All authors contributed equally to the conception, development, and writing of this manuscript. All authors have read and approved the final version.

\bmhead{Acknowledgement} 
Not applicable.

\bmhead{Human Participants and/or Animals} 
This study does not include any research involving human participants or animals.




\end{document}